\documentclass[11pt]{article}
\usepackage{graphicx}
\usepackage{longtable,lscape}
\usepackage{amsthm}
\usepackage{amsfonts}
\usepackage{amsmath}
\usepackage{bbm}
\usepackage{float}
\usepackage{url}
\newtheorem{definition}{Definition}
\newtheorem{theorem}{Theorem}

\newcommand{\captionfonts}{\footnotesize}
\makeatletter  
\long\def\@makecaption#1#2{%
  \vskip\abovecaptionskip
  \sbox\@tempboxa{{\captionfonts #1: #2}}%
  \ifdim \wd\@tempboxa >\hsize
    {\captionfonts #1: #2\par}
  \else
    \hbox to\hsize{\hfil\box\@tempboxa\hfil}%
  \fi
  \vskip\belowcaptionskip}
\makeatother 
\begin{document}
\title{Quantum Entanglement in Concept Combinations}
\author{Diederik Aerts and Sandro Sozzo \vspace{0.5 cm} \\ 
        \normalsize\itshape
        Center Leo Apostel for Interdisciplinary Studies \\
        \normalsize\itshape
        Brussels Free University \\ 
        \normalsize\itshape
         Krijgskundestraat 33, 1160 Brussels, Belgium \\
        \normalsize
        E-Mails: \url{diraerts@vub.ac.be,ssozzo@vub.ac.be} \\
               }
\date{}
\maketitle
\begin{abstract}
\noindent
Research in the application of quantum structures to cognitive science confirms that these structures quite systematically appear in the dynamics of concepts and their combinations and quantum-based models faithfully represent experimental data of situations where classical approaches are problematical. In this paper, we analyze the data we collected in an experiment on a specific conceptual combination, showing that Bell's inequalities are violated in the experiment. We present a new refined entanglement scheme to model these data within standard quantum theory rules, where `entangled measurements and entangled evolutions' occur, in addition to the expected `entangled states', and present a full quantum representation in complex Hilbert space of the data. This stronger form of entanglement in measurements and evolutions might have relevant applications in the foundations of quantum theory, as well as in the interpretation of nonlocality tests. It could indeed explain some non-negligible `anomalies' identified in EPR-Bell experiments.
\end{abstract}

\section{Introduction\label{intro}}

Two years ago we started the study of the structure of the combination of two concepts, the concept {\it Animal}, and the concept {\it Acts}, in the sentence {\it The Animal Acts}, by means of experiments with human subjects \cite{as2011}. Inspired by the type of coincidence experiments done in physics on compound quantum systems, giving rise to the identification of entanglement in such compound quantum systems, our investigation of {\it The Animal Acts} employed similar coincidence experiments. In the statistics of the experimental data we collected, we identified a violation of Bell's inequalities, very resembling to the violations of this inequality found in quantum physics \cite{bell1964}, and announced this finding as `the identification of entanglement in concept combinations' \cite{as2011}. In the present article we put forward additional elements of this cognitive entanglement that we have investigated meanwhile in great detail, and construct a full quantum mechanical representation in complex Hilbert space of the experimental data. As we will make clear in the following, our experimental cognitive violation of Bell's inequality made us gain quite some new insights into the nature and understanding of entanglement situations violating Bell's inequality, also relevant for their interpretation in micro-physics. 

We mention shortly the scientific context in which this research takes place. For our Brussels group, the inspiration to search for the identification of quantum structures in human cognition followed from a general investigation of classical and quantum probability structures some decades ago. More specifically, we were inspired by the problem of hidden-variable theories in quantum physics, i.e. pondering the question whether classical probability can reproduce the predictions of quantum mechanics \cite{aerts1986,aerts1992,aerts1994,aerts1998b,aerts1999b}. Understanding the fundamental characteristics of classical and quantum probability led us to identify situations in the macroscopic world entailing structures that are characteristic of quantum systems \cite{aertsaerts1994,aertsbroekaertsmets1999a,aertsaertsbroekaertgabora2000}. It was in this mindset that we developed a quantum modeling scheme for concepts and their combinations \cite{aertsgabora2005a,aertsgabora2005b}. How concepts combine and interact with each other and with large pieces of text, remains indeed one of the important unsolved problems in cognitive science, and constitutes one of the missing cornerstones for a deeper understanding of human thought itself. A key element of this situation was identified firstly in the so-called `Guppy effect' by Osherson and Smith \cite{oshersonsmith81}, when they observed that for the concepts {\it Pet} and {\it Fish} and their conjunction {\it Pet-Fish}, while an exemplar such as {\it Guppy} is very typical of {\it Pet-Fish}, it is neither typical of {\it Pet} nor of {\it Fish}. This Guppy effect was studied again in Hampton's experiments \cite{hampton88a,hampton88b}, with a focus on `membership weight' instead of `typicality', and measuring the deviation from classical set (fuzzy set) membership weights of exemplars with respect to pairs of concepts and their conjunction or disjunction. However, it is generally accepted that none of the currently existing concepts theories provides a satisfactory description of such effects \cite{oshersonsmith02,komatsu01,fodor01,rips03}.

Inspired by our previous work on quantum axiomatics, we elaborated the `SCoP  formalism' to model concepts and cope with the above
problems \cite{aertsgabora2005a,aertsgabora2005b}. In the SCoP formalism, each concept is associated with sets of states, contexts and properties. Concepts change continuously under the influence of context and this change is described by a change of the state of the concept. The SCoP formalism, when complex Hilbert space is used to represent the states, entails `a quantum modeling scheme for concepts', and was as such successively applied to model concept combinations accounting for the guppy effect as arising from the above mentioned quantum effects, and explicitly modeling Hampton's experimental data \cite{as2011,aerts2009,aertsdhooghehaven2010,abgs2012,ags2012}. We finish this short oversight by mentioning that these findings naturally fit in the emerging domain of research called `quantum cognition', consisting in the application of quantum structures in cognitive science  (see, e.g., \cite{bruzaetal2007,bruzaetal2008,bruzaetal2009,pb2009,k2010,songetal2011,bpft2011,bb2012,busemeyeretal2012}).

The focus of the present paper is the quantum modeling of the experimentally collected data on the situation of two specific concepts, {\it Animal} and {\it Acts}, and their combination {\it The Animal Acts} \cite{as2011,ags2012}. We first again describe explicitly this experiment in Sect. \ref{experiment} where we also show how Bell's inequalities \cite{bell1964,chsh69} are violated by the measured expectation values. In Sections \ref{theoretical}, \ref{calculations} and \ref{EPR_Bell} we put forward the new elements of our investigation, i.e. a presentation of a generalised entanglement scheme, a full quantum complex Hilbert space modeling of the data, and possible influence of our findings on entanglement in physics.

\section{Description of the experiment\label{experiment}}
In 1964 John Bell proved that, if one introduces local realism as a reasonable hypothesis for a physical theory, then one can derive an inequality for the expectation values of suitable physical observables (`Bell's inequality') which is violated in quantum mechanics \cite{bell1964}. This violation is due to a feature of quantum mechanics which is called `entanglement'. But, the violation of Bell's inequalities also proves the impossibility to cast quantum probabilities into a unique classical Kolmogorovian space \cite{aerts1986,af1982,pitowsky1989}. One generally concludes that, because of entanglement, one cannot consider the component parts of a composite quantum entity separately, but the entity must be described as an undivided whole. 
We aim to illustrate in this section how entanglement appears in the combination of human concepts as due to the violation of Bell's inequalities. 

We regard the sentence {\it The Animal Acts} as a combination of the concepts {\it Animal} and {\it Acts}. Then, 
we consider two pairs of exemplars, or states, of {\it Animal}, namely {\it Horse}, {\it Bear} and {\it Tiger}, {\it Cat}, and two pairs of exemplars, or states, of {\it Acts}, namely {\it Growls}, {\it Whinnies} and {\it Snorts}, {\it Meows}. By {\it Acts} we thus mean the specific action of {\it Making A Sound}. We introduce the single experiments $A$ and $A'$ for the concept {\it Animal}, and $B$ and $B'$ for the concept {\it Acts}. Experiment $A$ consists in participants choosing between {\it Horse} and {\it Bear} answering the question `is a good example of' {\it Animal}. We put as outcome $\lambda_H=+1$ if {\it Horse} is chosen, hence the state of {\it Animal} changes to {\it Horse}, and $\lambda_B=-1$ if {\it Bear} is chosen, hence the state of {\it Animal} changes to {\it Bear}. Experiment $A'$ consists in participants choosing between {\it Tiger} and {\it Cat} answering the question `is a good example of' {\it Animal}. We consistently put $\lambda_T=+1$ if {\it Tiger} is chosen and $\lambda_C=-1$ if {\it Cat} is chosen. Experiment $B$ consists in participants choosing between {\it Growls} and {\it Whinnies} answering the question `is a good example of' {\it Acts}. We put $\lambda_G=+1$ if {\it Growls} is chosen and $\lambda_W=-1$ if {\it Whinnies} is chosen. Experiment $B'$ consists in participants choosing between {\it Snorts} and {\it Meows} answering the question `is a good example of' {\it Acts}. Again, we put $\lambda_S=+1$ if {\it Snorts} is chosen and $\lambda_M=-1$ if {\it Meows} is chosen. 

Let us now come to the coincidence experiments $AB$, $AB'$, $A'B$ and $A'B'$ for the conceptual combination {\it The Animal Acts}. In all experiments, we ask test subjects to answer the question `is a good example of' the concept {\it The Animal Acts}. In experiment $AB$, participants choose among the four possibilities (1) {\it The Horse Growls}, (2) {\it The Bear Whinnies} -- and if one of these is chosen we put $\lambda_{HG}=\lambda_{BW}=+1$ -- and (3) {\it The Horse Whinnies}, (4) {\it The Bear Growls} -- and if one of these is chosen we put $\lambda_{HW}=\lambda_{BG}=-1$. In experiment $AB'$, they choose among (1) {\it The Horse Snorts}, (2) {\it The Bear Meows} -- and in case one of these is chosen we put $\lambda_{HS}=\lambda_{BM}=+1$ -- and (3) {\it The Horse Meows}, (4) {\it The Bear Snorts} -- and in case one of these is chosen we put $\lambda_{HS}=\lambda_{BM}=-1$. In experiment $A'B$, they choose among (1) {\it The Tiger Growls}, (2) {\it The Cat Whinnies} -- and in case one of these is chosen we put $\lambda_{TG}=\lambda_{CW}=+1$ -- and (3) {\it The Tiger Whinnies}, (4) {\it The Cat Growls} -- and in case one of these is chosen we put $\lambda_{TW}=\lambda_{CG}=-1$. 
Finally, in experiment $A'B'$, participants choose among (1) {\it The Tiger Snorts}, (2) {\it The Cat Meows} -- and in case one of these is chosen we put $\lambda_{TS}=\lambda_{CM}=+1$ -- and (3) {\it The Tiger Meows}, (4) {\it The Cat Snorts} -- and in case one of these is chosen we put $\lambda_{TM}=\lambda_{CS}=-1$.

We evaluate now the expectation values $E(A,B)$, $E(A, B')$, $E(A', B)$ and $E(A', B')$ associated with the experiments $AB$, $AB'$, $A'B$ and $A'B'$ respectively, and insert the values into the Clauser-Horne-Shimony-Holt (CHSH) version of Bell's inequality  

\vspace{-1mm}
\begin{equation} \label{chsh}
-2 \le E(A',B')+E(A',B)+E(A,B')-E(A,B) \le 2
\end{equation}

\vspace{-1mm}
\noindent
\cite{chsh69}. 
We performed a concrete experiment involving 81 participants who were presented a questionnaire to be filled out in which they were asked to make a choice among the above alternatives in the experiments $A$, $B$, $A'$ and $B'$, and also $AB$, $AB'$, $A'B$ and $A'B'$. 
Table \ref{tab} contains the results of our experiment \cite{as2011}.

If we denote by $P(A_1,B_1)$, $P(A_2,B_2)$, $P(A_1,B_2)$, $P(A_2,B_1)$, the probability that {\it The Horse Growls}, {\it The Bear Whinnies},  
{\it The Horse Whinnies}, {\it The Bear Growls}, respectively, is chosen in the coincidence experiment $AB$, and so on in the other experiments, the expectation values are, in the large number limits, 

\vspace{-5mm}
\begin{eqnarray}
&E(A,B)=&P(A_1,B_1)+P(A_2,B_2)-P(A_2,B_1)-P(A_1,B_2)=-0.7778  \nonumber \\
&E(A',B)=&P(A'_1,B_1)+P(A'_2,B_2)-P(A'_2,B_1)-P(A'_1,B_2)=0.6543 \nonumber \\
&E(A,B')=&P(A_1,B'_1)+P(A_2,B'_2)-P(A_2,B'_1)-P(A_1,B'_2)=0.3580 \nonumber \\
&E(A',B')=&P(A'_1,B'_1)+P(A'_2,B'_2)-P(A'_2,B'_1)-P(A'_1,B'_2)=  0.6296 \nonumber
\end{eqnarray}

\vspace{-1mm}
\noindent
Hence, Eq. (\ref{chsh}) gives 
\begin{equation}
E(A',B')+E(A',B)+E(A,B')-E(A,B)=2.4197
\end{equation}
which is greater than 2. This entails that (i) it violates Bell's inequalities, and (ii) the violation is close the maximal possible violation in quantum theory, viz. $2\cdot\sqrt{2} \approx 2.8284$.
\begin{table} \label{tab}
\centering
\begin{footnotesize}
\begin{tabular}{|c |c | c | c| c| }
\hline
single \textrm{$A$}, \textrm{$B$}& \emph{Horse}& \emph{Bear}& \emph{Growls}& \emph{Whinnies}\\
 & $P(A_1)=0.5309$ &$P(A_2)=0.4691$ & $P(B_1)=0.4815$ & $P(B_2)=0.5185$ \\
\hline
single \textrm{$A'$}, \textrm{$B'$}& \emph{Tiger}& \emph{Cat}& \emph{Snorts}& \emph{Meows}\\
 & $P(A_1)=0.7284$ &$P(A_2)=0.2716$ & $P(B_1)=0.3210$ & $P(B_2)=0.6790$ \\
 \hline
coincidence \textrm{$AB$} & \emph{Horse Growls} & \emph{Horse Whinnies} & \emph{Bear Growls} & \emph{Bear Whinnies}\\
 & $P(A_1,B_1)=0.049$ & $P(A_1,B_2)=0.630$ & $P(A_2,B_1)=0.259$ & $P(A_2,B_2)=0.062$  \\
\hline
coincidence \textrm{$AB'$} & \emph{Horse Snorts} & \emph{Horse Meows} & \emph{Bear Snorts} & \emph{Bear Meows}\\
& $P(A_1,B'_1)=0.593$ & $P(A_1, B'_2)=0.025$ & $P(A_2,B'_1)=0.296$   & $P(A_2,B'_2)=0.086$ \\
\hline
coincidence \textrm{$A'B$} & \emph{Tiger Growls} & \emph{Tiger Whinnies} & \emph{Cat Growls} & \emph{Cat Whinnies}\\
 & $P(A'_1,B_1)=0.778$ & $P(A'_1, B_2)=0.086$ & $P(A'_2,B_1)=0.086$  &  $P(A'_2,B_2)=0.049$ \\
\hline
coincidence \textrm{$A'B'$} & \emph{Tiger Snorts} & \emph{Tiger Meows} & \emph{Cat Snorts} & \emph{Cat Meows}\\
 & $P(A'_1,B'_1)=0.148$ & $P(A'_1, B'_2)=0.086$ & $P(A'_2,B'_1)=0.099$ & $P(A'_2,B'_2)=0.667$\\
\hline
\end{tabular}
\caption{The data collected in the single and coincidence experiments on entanglement in concepts \cite{as2011}.}
\end{footnotesize}
\end{table}
Hence, the violation is very significant 
because it can be shown that effects of disturbance of the experiment push the value of the Bell expression in Eq. (\ref{chsh}) toward a value between -2 and +2. To see how small the probability is that the resulting violation of Bell's inequality would be due to chance, we calculated the p value with a single samples t-test against the value 2. We found p = 0.0171 manifestly below 0.05. Hence, the null hypothesis, i.e. that the value is in the interval [-2,+2], and only probability fluctuations would give us for this specific test of our 81 individuals the value 2.4197, is very improbable, namely smaller than 0.0171. Hence this null hypothesis should be rejected, and our identified violation considered as genuine.   

\section{A Quantum Representation in Complex Hilbert space\label{theoretical}}
Before we put forward our quantum representation in complex Hilbert space of the collected data, we need to analyse in depth some of the notions related to entanglement and the quantum formalism. Indeed, as it will show quickly, our experimental data force us to touch in a new and surprising way  at these basic notions, and use them in a more subtle than usual manner, remaining however completely compatible with the quantum formalism in its essence. To introduce this new way, let us first remark that the coincidence experiments $AB$, $AB'$, $A'B$ and $A'B'$, are experiments with four outcomes each, and four states each to collapse to when one of the outcomes is secured. This means that in essence their measurement and evolution dynamics is described in a four dimensional complex Hilbert space, following the standard rules of the quantum formalism. On the other hand, the experiments $A$, $B$, $A'$ and $B'$ are experiments with two outcomes each, and two states to collapse to when one of the outcomes is secured. This means that in essence their measurement and evolution dynamics is described in a two dimensional complex Hilbert space, following the standard rules of the quantum formalism. The above is equally so for the in physics well-known and for the notion of entanglement archetypical situation of two spin 1/2 quantum particles with entangled spins. We know that then, for this well-known physics situation of entangled spins, the four dimensional complex Hilbert space is taken to be isomorphic to the tensor product of the two dimensional complex Hilbert spaces. And, it is by means of this isomorphism that the spins are described as sub entities of the compound entity of coupled spins, and entanglement is identified in the states of the compound entity, which hence becomes an entity of two coupled spins. More specifically, entanglement is identified by the presence of non-product vectors in this tensor product, and it are these non-product vectors that also give rise to the violation of typical Bell-type inequalities.

Physicists, and other scientists, who have been working for many decades now mainly with this archetypical coupled spin entangled situation, have come to believe that what we describe above is all there is to, i.e. `one takes the tensor product for the set of states of the compound entity, and entanglement appears on the scene, expressed mathematically by the presence of non-product vectors in this tensor product. And, additionally, the violation of Bell's inequality expresses how such non-product vectors of the tensor product produce entanglement.' Our data, collected on the compound entity {\it The Animal Acts}, have forced us to see the too limited scope of the above described entanglement scheme. Moreover, the scheme is not too limited in the sense that a generalisation of standard quantum theory is at stake. No, standard quantum theory, when situations of entanglement are analysed carefully, contains a more general setting for entanglement in its full standard version. Hence, the limitation in the customary entanglement scheme is not due to standard quantum theory, but to taking the coupled spins as archetypes for all entanglement situations. What is more, we have good reasons to believe that also the physics situation of the coupled spins is in need of the more general scheme that we will present in the following for our data on {\it The Animal Acts}.

We start the introduction of our more general entanglement scheme with a first remark. Mathematically ${\mathbb C}^4$ is isomorphic to ${\mathbb C}^2\otimes{\mathbb C}^2$, since both are four dimensional complex Hilbert spaces, and hence isomorphism can immediately be inferred by just linking two orthonormal (ON) bases in both spaces. And ${\mathbb C}^4$, and hence also ${\mathbb C}^2\otimes{\mathbb C}^2$, models the states of the considered compound entity, whether this compound entity is the coupled spins, or the concept combination {\it The Animal Acts}. The measurements and dynamical evolutions however of these entities are represented by linear operators -- respectively self-adjoint ones and unitary ones -- of the respective Hilbert spaces within a standard quantum theory modeling. There is another less well-known mathematical isomorphic correspondence, namely the set of all linear operators of ${\mathbb C}^4$, let is denote it by $L({\mathbb C}^4)$, is isomorphic to the tensor product of the two sets of all linear operators of ${\mathbb C}^2$, denoted $L({\mathbb C}^2)$, i.e. $L({\mathbb C}^4)\cong L({\mathbb C}^2)\otimes L({\mathbb C}^2)$. From standard quantum theory point of view, and certainly taking into account its mathematical structure, there would hence be no reason at all to suppose that the self-adjoint operators representing measurements on the compound entity, or the unitary operators, representing dynamical evolutions of the compound entity, would not be entangled, where `entangled operators' mean `non product operators'. Product operators do pertain to $L({\mathbb C}^2)\otimes L({\mathbb C}^2)$, but they represent equally special and simple cases for the operators of $L({\mathbb C}^2)\otimes L({\mathbb C}^2)$ as product vectors are special and simple cases for the vectors of ${\mathbb C}^2\otimes{\mathbb C}^2$. The investigation of our `animal acts' experimental data has shown us that the above mentioned operator-linked entanglement is not just a mathematical artefact of quantum theory, but also appears in nature. More specifically, the statistical structure of the data force the presence of entanglement that cannot be expressed in the state alone by means of non product vectors, but needs non product operators -- self-adjoint non-product operators for the measurements and unitary non-product operators for the dynamical evolutions -- to be modeled. Moreover, there are serious reasons to believe that also for the coupled spins such a broader range of operator-linked entanglement is present. We emphasise the latter since, if true, it constitutes an intriguing example of how research in quantum cognition can enlighten situations in physics, we comment more on this in Sect. \ref{EPR_Bell}. Let us introduce now the necessary mathematical items and theorem to prove what we have stated so far.

We define the notions of `product state', `product measurement' and `product dynamical evolution' as we will use it in our entanglement scheme. For this we consider the general form of an isomorphism $I$ between ${\mathbb C}^4$ and ${\mathbb C}^2 \otimes {\mathbb C}^2$, by linking the elements of an ON basis $\{|x_1\rangle, |x_2\rangle, |x_3\rangle, |x_4\rangle \}$ of ${\mathbb C}^4$ to the elements $\{|c_1\rangle\otimes|d_1\rangle, |c_1\rangle\otimes |d_2\rangle, |c_2\rangle\otimes|d_1\rangle, |c_2\rangle\otimes|d_2\rangle\}$ of the type of ON basis of ${\mathbb C}^2 \otimes {\mathbb C}^2$ where $\{|c_1\rangle, |c_2\rangle\}$ and $\{|d_1\rangle, |d_2\rangle\}$ are ON bases of ${\mathbb C}^2$ each
\begin{eqnarray}
&&I: {\mathbb C}^4 \rightarrow {\mathbb C}^2 \otimes {\mathbb C}^2 \\
&&I|x_1\rangle=|c_1\rangle\otimes|d_1\rangle \quad I|x_2\rangle=|c_1\rangle\otimes |d_2\rangle \label{corr1}\\ 
&&I|x_3\rangle=|c_2\rangle\otimes|d_1\rangle \quad I|x_4\rangle=|c_2\rangle\otimes|d_2\rangle \label{corr2}
\end{eqnarray}
Let us indicate some general properties of an isomorphism between Hilbert spaces. Suppose we consider a linear operator $A$ in ${\mathbb C}^4$, then the image of this operator in ${\mathbb C}^2 \otimes {\mathbb C}^2$ through $I$ is given by $IAI^{-1}$. If $|w\rangle$ is an eigenvector of $A$ with eigenvalue $\lambda$, then $I|w\rangle$ is an eigenvector of $IAI^{-1}$ with the same eigenvalue $\lambda$. This follows right away from the calculation $IAI^{-1}(I|w\rangle)=IA|w\rangle=\lambda I|w\rangle$. An isomorphism conserves length and orthogonality of vectors. 
\begin{definition}
A state $p$ represented by the unit vector $|p\rangle \in {\mathbb C}^4$ is a `product state', with respect to $I$, if there exists two states $p_a$ and $p_b$, represented by the unit vectors $|p_a\rangle \in {\mathbb C}^2$ and $|p_b\rangle \in {\mathbb C}^2$, respectively, such that $I|p\rangle=|p_a\rangle\otimes|p_b\rangle$. Otherwise, $p$ is an `entangled state' with respect to $I$.
\end{definition}
\begin{definition}
A measurement $e$ represented by a self-adjoint operator ${\cal E}$ in ${\mathbb C}^4$ is a `product measurement', with respect to $I$, if there exists
measurements $e_a$ and $e_b$, represented by the self-adjoint operators  ${\cal E}_a$ and ${\cal E}_b$, respectively, in ${\mathbb C}^2$ such that $I{\cal E}I^{-1}={\cal E}_a \otimes {\cal E}_b$. Otherwise, $e$ is an `entangled measurement' with respect to $I$.
\end{definition}
\begin{definition}
A dynamical evolution $u$ represented by a unitary operator ${\cal U}$ in ${\mathbb C}^4$ is a `product evolution', with respect to $I$, if there exists
dynamical evolutions $u_a$ and $u_b$, represented by the unitary operators operators  ${\cal U}_a$ and ${\cal U}_b$, respectively, in ${\mathbb C}^2$ such that $I{\cal U}I^{-1}={\cal U}_a \otimes {\cal U}_b$. Otherwise, $u$ is an `entangled evolution' with respect to $I$.
\end{definition}
Remark that the notion of product states, measurements and evolutions, are defined with respect to the considered isomorphism between ${\mathbb C}^4$ and ${\mathbb C}^2 \otimes {\mathbb C}^2$. This expresses well the essence of the physical content of what entanglement is, namely a `name giving' to non-product structures appearing in an identification procedure of sub entities of a considered compound entity applying standard quantum theory. It is natural that such a name giving depends on the isomorphism considered in this sub entity identification procedure. This means more concretely that there is freedom relative to a unitary transformation for different identification procedures that are equivalent following standard quantum theory rules. That there are operator-linked types of entanglement, not reducible to customary state-linked entanglement, and that entanglement depends on the isomorphism identifying sub entities, are the crucial insights that were forced upon us during our struggle to model the data we collected from our `animal acts' experiments. That the way to identify sub entities of the compound entity is not unique following standard quantum theory results in an entanglement scheme substantially more general and elaborate that the customary known and used one. We continue by proving some necessary theorems.
\begin{theorem}\label{th1}
The spectral family of a self-adjoint operator ${\cal E}=I^{-1}{\cal E}_a \otimes {\cal E}_bI$ representing a product measurement with respect to $I$, has the form $\{I^{-1}|a_1\rangle \langle a_1|\otimes |b_1\rangle \langle b_1|I$, $I^{-1}|a_1\rangle \langle a_1| \otimes |b_2\rangle \langle b_2|I$, $I^{-1}|a_2\rangle \langle a_2| \otimes |b_1\rangle \langle b_1|I, I^{-1}|a_2\rangle \langle a_2| \otimes |b_2\rangle \langle b_2|I\}$, where $\{|a_1\rangle \langle a_1|, |a_2\rangle \langle a_2|\}$ is a spectral family of ${\cal E}_a$ and $\{|b_1\rangle \langle b_1|, |b_2\rangle \langle b_2|\}$ is a spectral family of ${\cal E}_b$. 
\end{theorem}
\begin{proof}
We have ${\cal E}_a=\lambda_1|a_1\rangle \langle a_1|+\lambda_2|a_2\rangle \langle a_2|$, ${\cal E}_b=\mu_1|b_1\rangle \langle b_1|+\mu_2|b_2\rangle \langle b_2|$, and hence
${\cal E}_a \otimes {\cal E}_b=(\lambda_1|a_1\rangle \langle a_1|+\lambda_2|a_2\rangle \langle a_2|) \otimes(\mu_1|b_1\rangle \langle b_1|+\mu_2|b_2\rangle \langle b_2|)=\lambda_1\mu_1 |a_1\rangle \langle a_1| \otimes |b_1\rangle \langle b_1|+\lambda_1\mu_2 |a_1\rangle \langle a_1| \otimes |b_2\rangle \langle b_2|+\lambda_2\mu_1 |a_2\rangle \langle a_2| \otimes |b_1\rangle \langle b_1|+\lambda_2\mu_2|a_2\rangle \langle a_2| \otimes |b_2\rangle \langle b_2|$.
From this follows that
\begin{eqnarray}
{\cal E}&=&\lambda_1\mu_1 I^{-1}|a_1\rangle \langle a_1| \otimes |b_1\rangle \langle b_1|I+\lambda_1\mu_2 I^{-1}|a_1\rangle \langle a_1| \otimes |b_2\rangle \langle b_2|I \nonumber \\
&&+\lambda_2\mu_1 I^{-1}|a_2\rangle \langle a_2| \otimes |b_1\rangle \langle b_1|I+\lambda_2\mu_2I^{-1}|a_2\rangle \langle a_2| \otimes |b_2\rangle \langle b_2|I
\end{eqnarray}
\end{proof}
\begin{theorem} \label{th2}
Let $p$ be a product state represented by the vector $|p\rangle=I^{-1}|p_a\rangle\otimes|p_b\rangle$ with respect to the isomorphism $I$, and $e$ a product measurement represented by the self-adjoint operator ${\cal E}=I{\cal E}_a \otimes {\cal E}_bI^{-1}$ with respect to the same $I$. Let $\{|y_1\rangle, |y_2\rangle, |y_3\rangle, |y_4\rangle\}$ be the ON basis of eigenvectors of ${\cal E}$, and $\{|a_1\rangle, |a_2\rangle\}$ and $\{|b_1\rangle, |b_2\rangle\}$ the ON bases of eigenvectors of ${\cal E}_a$ and ${\cal E}_b$ respectively. Then, we have  $p(A_1)+p(A_2)=p(B_1)+p(B_2)=1$, and  $p(Y_1)=p(A_1)p(B_1)$, $p(Y_2)=p(A_1)p(B_2)$, $p(Y_3)=p(A_2)p(B_1)$ and $p(Y_4)=p(A_2)p(B_2)$, where $\{p(Y_1), p(Y_2), p(Y_3), p(Y_4)\}$ are the probabilities to collapse to states $\{|y_1\rangle, |y_2\rangle, |y_3\rangle, |y_4\rangle\}$, and $\{p(A_1), p(A_2)\}$ and $\{p(B_1), p(B_2)\}$ are the probabilities to collapse to states $\{|a_1\rangle, |a_2\rangle\}$ and $\{|b_1\rangle, |b_2\rangle\}$ respectively.
\end{theorem}
\begin{proof}
Let us calculate, e.g., $p(Y_1)$. From Th. \ref{th1} follows that $p(Y_1)=\langle p| y_1\rangle\langle y_1|p\rangle$=$(\langle p_a|\otimes \langle p_b|I)I^{-1}|a_1\rangle\langle a_1| \otimes |b_1\rangle \langle b_1|I(I^{-1}|p_a\rangle\otimes|p_b\rangle)$=$(\langle p_a|\otimes\langle p_b|)|a_1\rangle\langle a_1| \otimes |b_1\rangle \langle b_1|(|p_a\rangle\otimes|p_b\rangle)=\langle p_a|a_1\rangle \langle a_1|p_a\rangle \langle p_b|b_1\rangle \langle b_1|p_b\rangle$=$p(A_1)p(B_1)$.
Analogously, we can prove the product rule for the remaining probabilities.
\end{proof}
With this theorem we prove that if there exists an isomorphism $I$ between ${\mathbb C}^4$ and ${\mathbb C}^2 \otimes {\mathbb C}^2$ such that state and measurement are both product with respect to this isomorphism, then the probabilities factorize.
A consequence is that in case the probabilities do not factorize, which is the case for the probabilities that we have measured in our experiment as exposed in Sect. \ref{experiment} -- indeed, for example, $p(Y_1)=0.0494\ne 0.2556=p(A_1)p(B_2)$ -- the theorem is not satisfied. This means that there does not exist an isomorphism between ${\mathbb C}^4$ and ${\mathbb C}^2 \otimes {\mathbb C}^2$ such that both state and measurement are product with respect to this isomorphism, and there is genuine entanglement. The above theorem however does not yet prove where this entanglement is located, and how it is structured. The next theorems tell us more about this.

We consider now the coincidence measurements $AB$, $AB'$, $A'B$ and $A'B'$ from a typical Bell-type experimental setting. For each measurement we consider the ON bases of its eigenvectors in ${\mathbb C}^4$. For the measurement $AB$ this gives rise to the unit vectors $\{|ab_{11}\rangle, |ab_{12}\rangle, |ab_{21}\rangle, |ab_{22}\rangle\}$, for $AB'$ to the vectors $\{|ab'_{11}\rangle, |ab'_{12}\rangle, |ab'_{21}\rangle, |ab'_{22}\rangle\}$, for $A'B$ to the unit vectors $\{|a'b_{11}\rangle, |a'b_{12}\rangle, |a'b_{21}\rangle, |a'b_{22}\rangle\}$ and for $A'B'$ to the vectors $\{|a'b'_{11}\rangle,|a'b'_{12}\rangle,$ $  |a'b'_{21}\rangle, |a'b'_{22}\rangle\}$. We introduce the dynamical evolutions $u_{AB'AB}$, \ldots, represented by the unitary operators ${\cal U}_{AB'AB}$,\ldots, connecting the different coincidence experiments for any combination of them  
\begin{equation}
{\cal U}_{AB'AB}:{\mathbb C}^4 \rightarrow {\mathbb C}^4 \ |ab_{11}\rangle \mapsto |ab'_{11}\rangle, |ab_{12}\rangle \mapsto |ab'_{12}\rangle, |ab_{21}\rangle \mapsto |ab'_{21}\rangle, |ab_{22}\rangle \mapsto |ab'_{22}\rangle 
\end{equation}
\begin{theorem}\label{measuremententanglement}
There exists a isomorphism between ${\mathbb C}^4$ and ${\mathbb C}^2 \otimes {\mathbb C}^2$ with respect to which both measurements $AB$ and $AB'$ are product measurements iff there exists an isomorphism between ${\mathbb C}^4$ and ${\mathbb C}^2 \otimes {\mathbb C}^2$ with respect to which the dynamical evolution $u_{AB'AB}$ is a product evolution and one of the measurements is a product measurement. In this case the marginal law is satisfied for the probabilities connected to these measurements, i.e. $p(A_1,B_1)+p(A_1,B_2)=p(A_1,B'_1)+p(A_1,B'_2)$.
\end{theorem}
\begin{proof}
Suppose that $AB$ and $AB'$ are product measurements with respect to $I$. This means that ${\cal E}_{AB}=I^{-1}{\cal E}_A\otimes{\cal E}_BI$ and ${\cal E}_{AB'}=I^{-1}{\cal E}_A\otimes{\cal E}_B'I$. Let us define
\begin{eqnarray}
{\cal U}_{BB'}: {\mathbb C}^2\rightarrow {\mathbb C}^2 \quad |b_1\rangle \mapsto |b'_1\rangle \quad |b_2\rangle \mapsto |b'_2\rangle
\end{eqnarray}
where $\{b_1, b_2\}$ and $\{b'_1,b'_2\}$ are the ON bases of eigenvectors of ${\cal E}_B$ and ${\cal E}_{B'}$ respectively.
We have 
\begin{eqnarray}
I^{-1}{\mathbb I} \otimes {\cal U}_{BB'}I|ab_{11}\rangle=I^{-1}{\mathbb I} \otimes {\cal U}_{BB'}(|a_1\rangle \otimes |b_1\rangle)=I^{-1}(|a_1\rangle \otimes |b'_1\rangle)=|ab'_{11}\rangle \\
I^{-1}{\mathbb I} \otimes {\cal U}_{BB'}I|ab_{12}\rangle=I^{-1}{\mathbb I} \otimes {\cal U}_{BB'}(|a_1\rangle \otimes |b_2\rangle)=I^{-1}(|a_1\rangle \otimes |b'_2\rangle)=|ab'_{12}\rangle \\
I^{-1}{\mathbb I} \otimes {\cal U}_{BB'}I|ab_{21}\rangle=I^{-1}{\mathbb I} \otimes {\cal U}_{BB'}(|a_2\rangle \otimes |b_1\rangle)=I^{-1}(|a_2\rangle \otimes |b'_1\rangle)=|ab'_{21}\rangle \\
I^{-1}{\mathbb I} \otimes {\cal U}_{BB'}I|ab_{22}\rangle=I^{-1}{\mathbb I} \otimes {\cal U}_{BB'}(|a_2\rangle \otimes |b_2\rangle)=I^{-1}(|a_2\rangle \otimes |b'_2\rangle)=|ab'_{22}\rangle
\end{eqnarray}
which proves that $I^{-1}{\mathbb I} \otimes {\cal U}_{BB'}I={\cal U}_{AB'AB}$. Hence $u_{AB'AB}$ is a product evolution.
Suppose now that $u_{AB'AB}$ is a product evolution with respect to $I$, and $AB$ is a product measurement, let us prove that $AB'$ is a product measurement. We have ${\cal E}_{AB}=I^{-1}{\cal E}_A\otimes{\cal E}_BI$ and ${\cal U}_{AB'AB}=I^{-1}{\mathbb I} \otimes {\cal U}_{BB'}I$. Consider
\begin{eqnarray}
|ab'_{11}\rangle \langle a b'_{11}|&=&{\cal U}_{AB'AB}|ab_{11}\rangle \langle ab_{11}|{\cal U}_{AB'AB}^{-1}={\cal U}_{AB'AB}I^{-1}|a_1\rangle \langle a_1| \otimes |b_1\rangle \langle b_1|I{\cal U}_{AB'AB}^{-1} \nonumber \\
&=&(I^{-1}{\mathbb I} \otimes {\cal U}_{BB'}I)I^{-1}|a_1\rangle \langle a_1| \otimes |b_1\rangle \langle b_1|I(I^{-1}{\mathbb I} \otimes {\cal U}_{BB'}I)^{-1} \nonumber \\
&=&I^{-1}{\mathbb I} \otimes {\cal U}_{BB'}(|a_1\rangle \langle a_1| \otimes |b_1\rangle \langle b_1|)I(I^{-1}({\mathbb I} \otimes {\cal U}_{BB'})^{-1}I) \nonumber \\
&=&I^{-1}{\mathbb I} \otimes {\cal U}_{BB'}(|a_1\rangle \langle a_1| \otimes |b_1\rangle \langle b_1|)({\mathbb I} \otimes {\cal U}_{BB'})^{-1}I) \nonumber \\
&=&I^{-1}(|a_1\rangle \langle a_1| \otimes |b'_1\rangle \langle b'_1|)I
\end{eqnarray}
In an analogous way we prove that 
\begin{eqnarray}
|ab'_{12}\rangle \langle a b'_{12}|=I^{-1}(|a_1\rangle \langle a_1| \otimes |b'_2\rangle \langle b'_2|)I \\
|ab'_{21}\rangle \langle a b'_{21}|=I^{-1}(|a_2\rangle \langle a_2| \otimes |b'_1\rangle \langle b'_1|)I \\
|ab'_{22}\rangle \langle a b'_{22}|=I^{-1}(|a_2\rangle \langle a_2| \otimes |b'_2\rangle \langle b'_2|)I
\end{eqnarray}
This proves that $AB'$ is a product measurement.

Let us prove now that the marginal law is satisfied for the probabilities connected to two product measurements. We have
\begin{eqnarray}
p(A_1,B_1)+p(A_1,B_2)&=&\langle p |ab_{11}\rangle\langle ab_{11}|p\rangle+ \langle p |ab_{12}\rangle\langle ab_{12}|p\rangle \nonumber\\
&=&\langle p| (I^{-1}|a_1\rangle\langle a_1| \otimes|b_1\rangle\langle b_1|I) |p \rangle+ \langle p|(I^{-1}|a_1\rangle\langle a_1|\otimes |b_2\rangle\langle b_2|I)|p\rangle \nonumber\\
&=&\langle p| (I^{-1}|a_1\rangle\langle a_1|) \otimes (|b_1\rangle\langle b_1|+|b_2\rangle\langle b_2)I |p\rangle\nonumber \\
&=&\langle p| (I^{-1}|a_1\rangle\langle a_1|) \otimes (|b'_1\rangle\langle b'_1|+|b'_2\rangle\langle b'_2|)I |p\rangle \nonumber\\
&=&\langle p| (I^{-1}|a_1\rangle\langle a_1| \otimes|b'_1\rangle\langle b'_1|I) |p \rangle+ \langle p|(I^{-1}|a_1\rangle\langle a_1|\otimes |b'_2\rangle\langle b'_2|I)|p\rangle\nonumber \\ 
&=&\langle p |ab'_{11}\rangle\langle ab'_{11}|p\rangle+ \langle p |ab'_{12}\rangle\langle ab'_{12}|p\rangle \nonumber \\
&=&p(A_1,B'_1)+p(A_1,B'_2)
\end{eqnarray}
\end{proof}
The above theorem introduces an essential deviation of the customary entanglement scheme, which we felt forced to consider as a consequence of our `animal acts' experimental data. Indeed, considering Table \ref{tab}, we have $P(A_1, B_1)+p(A_1,B_2)=0.679\ne0.618=p(A_1,B'_1)+p(A_1,B'_2)$, which shows that the marginal law is not satisfied for our data. Remark that $p(A'_1,B_1)+p(A'_1,B_2)=0.864\ne0.234=p(A'_1,B'_1)+p(A'_1,B'_2)$ showing that for this case the deviation from the marginal law of our data is very strong, and cannot be considered to be due to experimental error.

Hence, for our data there does not exist an isomorphism between 
${\mathbb C}^4$ and ${\mathbb C}^2\otimes{\mathbb C}^2$, such that with respect to this isomorphism all measurements that we performed in our experiment can be considered to be product measurements. It right away shows that we will not able to model our data within the customary entanglement scheme used for the coupled spins in ${\mathbb C}^2\otimes{\mathbb C}^2$. We could have expected this, since indeed, in this customary scheme all considered measurements are product measurements, and entanglement only appears in the state of the compound entity. Does this mean that amongst our measurements performed in our experiment at least one will appear as an entangled measurement in any modeling we can propose? Indeed, no escape is possible for this conclusion, at least if we want to define entanglement with respect to one fixed isomorphism between ${\mathbb C}^4$ and ${\mathbb C}^2\otimes{\mathbb C}^2$.

Before we continue our analysis we want to reflect somewhat about the state of affairs we here identify. Entanglement being present at the level of the measurements seems to be far more drastic than entanglement being present at the level of the state. Indeed, like we mentioned already, entanglement at the level of the state is customary interpreted as due to not being able any longer to consider the compound entity as consisting of two well defined sub entities, and needing to consider it as a new undivided whole. Hence, entanglement at the level of the measurements, does this mean then that not only the entity is a new undivided whole, but also the measurements are connected in an undivided way? Of course, it is true, our measurements have been performed by subjects, using their undivided mind to choose, given the state of the combined concept {\it The Animal Acts}, and perhaps our finding is an expression of this? On the other hand, we are confronted with the following. Even if in our experiment the undivided mind of the subject has made the choice on the combination {\it The Animal Acts}, actually the choice is made between four well defined combinations of states, and each state of the combination is a state of {\it Animal} or of {\it Acts} separately. The situation becomes even more intricate if we note that an entangled measurement cannot give rise to collapsed states that are product states. Does this then mean that some of the collapsed states, such as {\it The Horse Whinnies}, or {\it The Cat Growls} etc\ldots, are `not' product states, but entangled states? Definitely there is still an aspect of our new entanglement scheme that has not been understood completely, and the next theorem is aimed at doing so. 

\begin{theorem}\label{unitaryentanglement}
For each coincidence measurement $AB$, there exists a specific isomorphism 
\begin{eqnarray}
&&I_{AB}: {\mathbb C}^4 \rightarrow {\mathbb C}^2 \otimes {\mathbb C}^2 \\
&&I_{AB}|ab_{11}\rangle=|a_1\rangle\otimes|b_1\rangle \quad I_{AB}|ab_{12}\rangle=|a_1\rangle\otimes |b_2\rangle \label{isomorphismI_AB} \\ 
&&I_{AB}|ab_{21}\rangle=|a_2\rangle\otimes|b_1\rangle \quad I_{AB}|ab_{22}\rangle=|a_2\rangle\otimes|b_2\rangle \label{isomorphismI_ABbis}
\end{eqnarray}
for which all entanglement related to this state-measurement situation is concentrated in the state $p$ of the considered compound entity in the form of a typical entangled state $I_{AB}|p\rangle$, where $|p\rangle$ is the vector representing $p$, which we call the $AB$-entanglement representation of $p$. The measurement $AB$ is a product measurements with respect to $I_{AB}$. Two such product measurements $AB$, with respect to $I_{AB}$, and $AB'$, with respect to $I_{AB'}$, are connected by product unitary transformations $I_{AB'}{\cal U}_{AB'AB}I^{-1}_{AB}$, such that for the measurements and the dynamical evolution connections in ${\mathbb C}^2 \otimes {\mathbb C}^2$ all is reduced to the customary entanglement scheme. However, two different entanglement representations of the state $|p\rangle$, e.g. $I_{AB}|p\rangle$ and $I_{AB'}|p\rangle$, are not representing the same state in ${\mathbb C}^2 \otimes {\mathbb C}^2$, i.e. are not vectors that differ only a phase factor, if the marginal law with respect to the two considered coincidence measurements $AB$ and $AB'$ is not satisfied.
\end{theorem}
\begin{proof}
The self-adjoint operator ${\mathcal E}_{AB}$ modeling $AB$ can always be written as follows 
${\mathcal E}_{AB}=\lambda_{11}|ab_{11}\rangle \langle ab_{11}|+\lambda_{12}|ab_{12}\rangle \langle ab_{12}|+\lambda_{21}|ab_{21}\rangle \langle ab_{21}|+\lambda_{22}|ab_{22}\rangle \langle ab_{22}|$, 
with $\lambda_{11}, \lambda_{12}, \lambda_{21}$ and $\lambda_{22}$ real numbers chosen such that for real numbers $\mu_1, \mu_2$ we have $\nu_1, \nu_2$, and $\lambda_{11}=\mu_1\nu_1$, $\lambda_{12}=\mu_1\nu_2$, $\lambda_{21}=\mu_2\nu_1$ and $\lambda_{22}=\mu_2\nu_2$. Hence
\begin{eqnarray}
I_{AB}{\mathcal E}_{AB}I_{AB}^{-1}&=&\mu_1\nu_1 |a_1\rangle \langle a_1|\otimes |b_1\rangle\langle b_1|+\mu_1\nu_2|a_1\rangle \langle a_1|\otimes |b_2\rangle\langle b_2| \nonumber \\
&&+\mu_2\nu_1|a_2\rangle \langle a_2|\otimes |b_1\rangle\langle b_1|+\mu_2\nu_2|a_2\rangle \langle a_2|\otimes |b_2\rangle\langle b_2| \nonumber \\
&=&(\mu_1|a_1\rangle \langle a_1|+\mu_2|a_2\rangle \langle a_2|)\otimes(\nu_1|b_1\rangle \langle b_1|+\nu_2|b_2\rangle \langle b_2|) \nonumber \\
&=&{\mathcal E}_A\otimes{\mathcal E}_B
\end{eqnarray}
which proves that $AB$ is a product measurement. We have
\begin{eqnarray}
I_{AB'}{\cal U}_{AB'AB}I^{-1}_{AB}{\mathcal E}_A\otimes{\mathcal E}_B(I_{AB'}{\cal U}_{AB'AB}I^{-1}_{AB})^{-1}&=&I_{AB'}{\cal U}_{AB'AB}I^{-1}_{AB}{\mathcal E}_A \otimes{\mathcal E}_BI_{AB}{\cal U}_{AB'AB}^{-1}I^{-1}_{AB'} \nonumber \\
&=&I_{AB'}{\cal U}_{AB'AB}{\cal E}_{AB}{\cal U}_{AB'AB}^{-1}I^{-1}_{AB'}=I_{AB'}{\cal E}_{AB'}I^{-1}_{AB'} \nonumber \\
&=&{\cal E}_A\otimes{\cal E}_{B'}
\end{eqnarray}
which proves that $I_{AB'}{\cal U}_{AB'AB}I^{-1}_{AB}$ works as a product unitary transformation ${\mathbb I}\otimes{\cal U}_{B'B}$, transforming ${\mathcal E}_A\otimes{\mathcal E}_B$ into ${\mathcal E}_A\otimes{\mathcal E}_B'$. 
Let us look now at the vectors $I_{AB}|p\rangle$ and $I_{AB'}|p\rangle$. We have 
\begin{eqnarray}
I_{AB}|p\rangle&=&I_{AB}|ab_{11}\rangle \langle ab_{11}|p\rangle +I_{AB}|ab_{12}\rangle \langle ab_{12}|p\rangle+ I_{AB}|ab_{21}\rangle \langle ab_{21}|p\rangle+I_{AB}|ab_{22}\rangle \langle ab_{22}|p\rangle \nonumber \\
&=&|a_1\rangle \otimes |b_1\rangle \langle ab_{11}|p\rangle +|a_1\rangle \otimes |b_2\rangle \langle ab_{12}|p\rangle+ |a_2\rangle \otimes |b_1\rangle \langle ab_{21}|p\rangle+|a_2\rangle \otimes |b_2\rangle \langle ab_{22}|p\rangle \\
I_{AB'}|p\rangle&=&I_{AB'}|ab'_{11}\rangle \langle ab'_{11}|p\rangle +I_{AB'}|ab'_{12}\rangle \langle ab'_{12}|p\rangle+ I_{AB'}|ab'_{21}\rangle \langle ab'_{21}|p\rangle+I_{AB'}|ab'_{22}\rangle \langle ab'_{22}|p\rangle \nonumber \\
&=&|a_1\rangle \otimes |b'_1\rangle \langle ab'_{11}|p\rangle +|a_1\rangle \otimes |b'_2\rangle \langle ab'_{12}|p\rangle+ |a_2\rangle \otimes |b'_1\rangle \langle ab'_{21}|p\rangle+|a_2\rangle \otimes |b'_2\rangle \langle ab'_{22}|p\rangle \nonumber  \\
&=&|a_1\rangle \otimes (|b_1\rangle \langle b_1|b'_1\rangle+|b_2\rangle \langle b_2|b'_1\rangle) \langle ab'_{11}|p\rangle +|a_1\rangle \otimes (|b_1\rangle \langle b_1|b'_2\rangle+|b_2\rangle \langle b_2|b'_2\rangle) \langle ab'_{12}|p\rangle \nonumber \\
&&+ |a_2\rangle \otimes (|b_1\rangle \langle b_1|b'_1\rangle+|b_2\rangle \langle b_2|b'_1\rangle) \langle ab'_{21}|p\rangle+|a_2\rangle \otimes (|b_1\rangle \langle b_1|b'_2\rangle+|b_2\rangle \langle b_2|b'_2\rangle) \langle ab'_{22}|p\rangle \nonumber \\
&=&|a_1\rangle \otimes |b_1\rangle (\langle b_1|b'_1\rangle\langle ab'_{11}|p\rangle+\langle b_1|b'_2\rangle\langle ab'_{12}|p\rangle)+|a_1\rangle \otimes |b_2\rangle (\langle b_2|b'_1\rangle\langle ab'_{11}|p\rangle+\langle b_2|b'_2\rangle\langle ab'_{12}|p\rangle) \nonumber \\
&&+|a_2\rangle \otimes |b_1\rangle (\langle b_1|b'_1\rangle\langle ab'_{21}|p\rangle+\langle b_1|b'_2\rangle\langle ab'_{22}|p\rangle)+|a_2\rangle \otimes |b_2\rangle (\langle b_2|b'_1\rangle\langle ab'_{21}|p\rangle+\langle b_2|b'_2\rangle\langle ab'_{22}|p\rangle)
\end{eqnarray}
For $I_{AB}|p\rangle$ to be equal to $I_{AB'}|p\rangle$ except for a phase factor $e^{i\alpha}$, we need the following equalities to hold
\begin{eqnarray}
\langle b_1|b'_1\rangle\langle ab'_{11}|p\rangle+\langle b_1|b'_2\rangle\langle ab'_{12}|p\rangle=\langle ab_{11}|p\rangle e^{i\alpha} \\
\langle b_2|b'_1\rangle\langle ab'_{11}|p\rangle+\langle b_2|b'_2\rangle\langle ab'_{12}|p\rangle=\langle ab_{12}|p\rangle e^{i\alpha} \\
\langle b_1|b'_1\rangle\langle ab'_{21}|p\rangle+\langle b_1|b'_2\rangle\langle ab'_{22}|p\rangle=\langle ab_{21}|p\rangle e^{i\alpha} \\
\langle b_2|b'_1\rangle\langle ab'_{21}|p\rangle+\langle b_2|b'_2\rangle\langle ab'_{22}|p\rangle=\langle ab_{22}|p\rangle e^{i\alpha}
\end{eqnarray}
This means that
\begin{eqnarray}
p(A_1,B_1)+p(A_1,B_2)&=&\langle p|ab_{11}\rangle \langle ab_{11}|p\rangle+\langle p|ab_{12}\rangle \langle ab_{12}|p\rangle \nonumber \\ 
&=&(\langle p|ab'_{11}\rangle\langle b'_1|b_1\rangle+\langle p|ab'_{12}\rangle\langle b'_2|b_1\rangle)(\langle b_1|b'_1\rangle\langle ab'_{11}|p\rangle+\langle b_1|b'_2\rangle\langle ab'_{12}|p\rangle) \nonumber \\
&&+(\langle p|ab'_{11}\rangle\langle b'_1|b_2\rangle+\langle p|ab'_{12}\rangle\langle b'_2|b_2\rangle)(\langle b_2|b'_1\rangle\langle ab'_{11}|p\rangle+\langle b_2|b'_2\rangle\langle ab'_{12}|p\rangle) \nonumber \\
&=&\langle p|ab'_{11}\rangle\langle b'_1|b_1\rangle\langle b_1|b'_1\rangle\langle ab'_{11}|p\rangle+\langle p|ab'_{11}\rangle\langle b'_1|b_1\rangle\langle b_1|b'_2\rangle\langle ab'_{12}|p\rangle \nonumber \\
&&+\langle p|ab'_{12}\rangle\langle b'_2|b_1\rangle\langle b_1|b'_1\rangle\langle ab'_{11}|p\rangle+\langle p|ab'_{12}\rangle\langle b'_2|b_1\rangle\langle b_1|b'_2\rangle\langle ab'_{12}|p\rangle \nonumber \\
&&+\langle p|ab'_{11}\rangle\langle b'_1|b_2\rangle\langle b_2|b'_1\rangle\langle ab'_{11}|p\rangle+\langle p|ab'_{11}\rangle\langle b'_1|b_2\rangle\langle b_2|b'_2\rangle\langle ab'_{12}|p\rangle \nonumber \\
&&+\langle p|ab'_{12}\rangle\langle b'_2|b_2\rangle\langle b_2|b'_1\rangle\langle ab'_{11}|p\rangle+\langle p|ab'_{12}\rangle\langle b'_2|b_2\rangle\langle b_2|b'_2\rangle\langle ab'_{12}|p\rangle \nonumber \\
&=&\langle p|ab'_{11}\rangle\langle b'_1|(|b_1\rangle\langle b_1|+|b_2\rangle\langle b_2|)|b'_1\rangle\langle ab'_{11}|p\rangle \nonumber \\
&&+\langle p|ab'_{11}\rangle\langle b'_1|(|b_1\rangle\langle b_1|+|b_2\rangle\langle b_2|)|b'_2\rangle\langle ab'_{12}|p\rangle \nonumber \\
&&+\langle p|ab'_{12}\rangle\langle b'_2|(|b_1\rangle\langle b_1|+|b_2\rangle\langle b_2|)|b'_1\rangle\langle ab'_{11}|p\rangle \nonumber \\
&&+\langle p|ab'_{12}\rangle\langle b'_2|(|b_1\rangle\langle b_1|+|b_2\rangle\langle b_2|)|b'_2\rangle\langle ab'_{12}|p\rangle \nonumber \\
&=&\langle p|ab'_{11}\rangle\langle b'_1|(|b'_1\rangle\langle b'_1|+|b'_2\rangle\langle b'_2|)|b'_1\rangle\langle ab'_{11}|p\rangle \nonumber \\
&&+\langle p|ab'_{11}\rangle\langle b'_1|(|b'_1\rangle\langle b'_1|+|b'_2\rangle\langle b'_2|)|b'_2\rangle\langle ab'_{12}|p\rangle \nonumber \\
&&+\langle p|ab'_{12}\rangle\langle b'_2|(|b'_1\rangle\langle b'_1|+|b'_2\rangle\langle b'_2|)|b'_1\rangle\langle ab'_{11}|p\rangle \nonumber \\
&&+\langle p|ab'_{12}\rangle\langle b'_2|(|b'_1\rangle\langle b'_1|+|b'_2\rangle\langle b'_2|)|b'_2\rangle\langle ab'_{12}|p\rangle \nonumber \\
&=&\langle p|ab'_{11}\rangle\langle b'_1|b'_1\rangle\langle b'_1|b'_1\rangle\langle ab'_{11}|p\rangle+\langle p|ab'_{12}\rangle\langle b'_2|b'_2\rangle\langle b'_2|b'_2\rangle\langle ab'_{12}|p\rangle \nonumber \\
&=&\langle p|ab'_{11}\rangle \langle ab'_{11}|p\rangle+\langle p|ab'_{12}\rangle\langle ab'_{12}|p\rangle \nonumber \\
&=&p(A_1,B'_1)+p(A_1,B'_2)
\end{eqnarray}
which shows that the marginal law is satisfied in that case. So, we have proven that in case the marginal law is not satisfied for $AB$ and $AB'$, the vectors $I_{AB}|p\rangle$ and $I_{AB'}|p\rangle$ do not represent the same states in ${\mathbb C}^2 \otimes {\mathbb C}^2$.
\end{proof}
From the above theorem follows that for a typical Bell type experimental situation, there are four distinct isomorphisms of ${\mathbb C}^{4}$ with the tensor product ${\mathbb C}^{2}\otimes{\mathbb C}^{2}$, namely $I_{AB}$, $I_{AB'}$, $I_{A'B}$ and $I_{A'B'}$, depending the ON basis of ${\mathbb C}^{4}$ we identify with a typical ON basis of this tensor product.
Each one of them defines for the state $p$ of the compound entity a typical entangled state $I_{AB}|p\rangle$ (respectively $I_{AB'}|p\rangle$, $I_{A'B}|p\rangle$, and $I_{A'B'}|p\rangle$), which we call the $AB$-entanglement (respectively $AB'$, $A'B$, $A'B'$) representation of $p$. These entanglement representations absorb all of the entanglement in the state. Indeed, in a complete analogous way as proved for $AB$ in theorem \ref{unitaryentanglement}, we prove that $AB'$, $A'B$, and $A'B'$ are product measurements with respect to respectively $I_{AB'}$, $I_{A'B}$ and $I_{A'B'}$.
However, for example, $AB'$ is in general an entanglement measurements with respect to $I_{AB}$, and this counts for all arbitrary elections of one of the coincidence experiment and one of the isomorphisms. Moreover, if we confine ourselves to the tensor product space ${\mathbb C}^{2}\otimes{\mathbb C}^{2}$, where all measurements now are product measurements, also the dynamical evolutions connecting the measurements are products. However, although all entanglement is pushed in the state, a new price has to be paid, i.e. the representations of the state in ${\mathbb C}^{2}\otimes{\mathbb C}^{2}$ by the vectors $I_{AB}|p\rangle$, $I_{AB'}|p\rangle$, $I_{A'B}|p\rangle$ and $I_{A'B'}|p\rangle$ are different for each one of the coincidence measurements $AB$, $AB'$, $A'B$ and $A'B'$, in case the marginal law is violated. This means that ${\mathbb C}^{2}\otimes{\mathbb C}^{2}$ has become literally a contextual representation, where contextual states appear for each of the considered measurements, all entanglement being represented in these states.

We have now all elements available to sharply describe the new entanglement scheme that we have introduced. {\it Whenever the marginal law related to two coincidence measurements is violated, these two coincidence experiments cannot be product measurements with respect to one and the same isomorphism with the tensor product space, where entanglement is identified. Hence, at least one of them will be an entangled measurement if we insist on a description making use of one isomorphism. The dynamical evolution that switches between two measurements is an entangled evolution if the marginal law is not satisfied between these two measurements.
It is possible to adopt contextually different entanglement identifications, choosing a specific isomorphism for each coincidence experiments. Then all the entanglement can be pushed into the state, and the measurements and dynamical evolutions become products. However, the description in the tensor product space becomes explicitly contextual, in the sense that for each coincidence experiment a different state is used to represent the compound entity.}

We believe that when we consider an entanglement situation we are tempted to introduce the contextually defined different isomorphisms that make all measurements products again, also if the marginal law is violated, because our focus is on the end states after the measurements, rather than on the physical state space before the measurements. Indeed, we are tempted to name the end states of a coincidence measurement by using names of end states of single measurements. For example, we use {\it Horse} and {\it Whinnies}, which are end states of the single measurements on {\it Animal} and {\it Acts}, to indicate the end state {\it The Horse Whinnies} of the compound entity {\it The Animal Acts} collapsed to by the coincidence measurement. Hence, in human language, we are tempted to choose for the contextual representation, with measurements and evolutions remaining products, and states becoming contextual. We believe that also in physics this temptation exists, and comment on this in Sect. \ref{EPR_Bell}. In the following section we provide an explicit representation of our experimental data in complex Hilbert space.
 
\section{Entangled measurements and their representation\label{calculations}}
In this section we calculate explicitly the self-adjoint operators ${\cal E}_{AB}$, ${\cal E}_{AB'}$, ${\cal E}_{A'B}$ and ${\cal E}_{A'B'}$ representing the coincidence measurements $AB$, $AB'$, $A'B$ and $A'B'$, with the data that we collected in our experiments on {\it The Animal Acts}.

First we choose a representation for the state $p$ of {\it The Animal Acts} by the unit vector $|p\rangle=|ae^{i\alpha}, be^{i\beta},ce^{i\gamma},de^{i\delta}\rangle$ in the canonical base of ${\mathbb C}^{4}$, with $a=0.23$, $b=0.62$, $c=0.75$, $d=0$, $\alpha=13.93^\circ$, $\beta=16.72^\circ$, $\gamma=9.69^\circ$, $\delta=194.15^\circ$. This choice of $|p\rangle$ is not arbitrary. 
The state was found by comparing our data with a tensor product representation only using product measurements, and by means of a numerical optimization procedure. This means that it is a unit vector closets to a possible representation with only product measurements that we could identify numerically.
We do not describe this procedure here, since it is not the focus of this article, but details can be obtained with the authors for those interested.

Let us construct now ${\cal E}_{AB}$. 
We express the ON basis of eigenvectors of ${\cal E}_{AB}$ in the 
canonical basis of ${\mathbb C}^{4}$ as 
$|ab_{11}\rangle=|x_1 e^{i \theta_1}, y_1 e^{i\phi_1}, z_1 e^{i \chi_1}, t_1 e^{i \mu_1}\rangle$, 
$|ab_{12}\rangle=|x_2 e^{i \theta_2}, y_2 e^{i\phi_2},$ $z_2 e^{i \chi_2}, t_2 e^{i \mu_2}\rangle$, 
$|ab_{21}\rangle=|x_3 e^{i \theta_3}, y_3 e^{i\phi_3}, z_3 e^{i \chi_3}, t_3 e^{i \mu_3}\rangle$, and 
$|ab_{22}\rangle=|x_4 e^{i \theta_4}, y_4 e^{i\phi_4}, z_4 e^{i \chi_4},$ $t_4 e^{i \mu_4}\rangle$. In the state $p$ of {\it The Animal Acts}, represented by the unit vector $|p\rangle$, one has the following probabilities 
$p(A_1,B_1)=\langle p|ab_{11}\rangle\langle ab_{11}|p\rangle=|\langle ab_{11}|p\rangle|^{2}$, $p(A_1,B_2)=\langle p|ab_{12}\rangle\langle ab_{12}|p\rangle=|\langle ab_{12}|p\rangle|^{2}$, $p(A_2,B_1)=\langle p|ab_{21}\rangle\langle ab_{21}|p\rangle=|\langle ab_{21}|p\rangle|^{2}$, $p(A_2,B_2)=\langle p|ab_{22}\rangle\langle ab_{22}|p\rangle=|\langle ab_{22}|p\rangle|^{2}$.
The following is a solution
\begin{eqnarray}
&|ab_{11}\rangle=&|0 e^{i 71.38^\circ}, 0.15 e^{i26.63^\circ}, 0.17 e^{i 30.07^\circ}, 0.97 e^{i 263.57^\circ} \rangle\\
&|ab_{12}\rangle=&|0.09 e^{i 149.21^\circ}, 0.96 e^{i322.93^\circ}, 0.96 e^{i 327.81^\circ}, 0.19 e^{i 21.14^\circ} \rangle\\
&|ab_{21}\rangle=&|0.12 e^{i 13.34^\circ}, 0.97 e^{i5.44^\circ}, 0.17 e^{i 189.97^\circ}, 0.12 e^{i 62.08^\circ} \rangle\\
&|ab_{22}\rangle=&|0.99 e^{i 71.17^\circ}, 0.11 e^{i242.99^\circ}, 0.11 e^{i 69.46^\circ}, 0 e^{i 125.70^\circ} \rangle
\end{eqnarray}
which exactly reproduces the experimental data in Tab. 2. The self-adjoint operator 
${\cal E}_{AB}$ is given by
\begin{eqnarray}
{\cal E}_{AB}&=&|ab_{11}\rangle\langle ab_{11}|-|ab_{12}\rangle\langle ab_{12}|-|ab_{21}\rangle\langle ab_{21}|+|ab_{22}\rangle\langle ab_{22}| \nonumber\\
&=&\small
\left( \begin{array}{cccc}
0.952 & -0.207-0.030i	&	0.224+0.007i & 0.003-0.006i \\				
-0.207+0.030i & -0.930	&	0.028-0.001i	&	-0.163+0.251i \\				
0.224-0.007i & 0.028+0.001i & -0.916 & -0.193+0.266i \\
0.003+0.006i & -0.163-0.251i & -0.193-0.266i & 0.895				
\end{array} \right)
\normalsize
\end{eqnarray}
where we assumed that $\lambda_{HG}=\lambda_{BW}=1$ and $\lambda_{HW}=\lambda_{BG}=-1$
with the aim of directly measuring the expectation value. 
In a completely analogous way, we can construct the other self-adjoint operators. In the second measurement 
$AB'$, the ON basis of eigenvectors is given by
\begin{eqnarray}
&|ab'_{11}\rangle=&|0.65 e^{i 69.64^\circ}, 0.48 e^{i38.08^\circ}, 0.45 e^{i 31.37^\circ}, 0.37 e^{i 269.21^\circ} \rangle\\
&|ab'_{12}\rangle=&|0.11 e^{i 207.96^\circ}, 0.63 e^{i208.97^\circ}, 0.77 e^{i 18.61^\circ}, 0.05 e^{i 205.71^\circ} \rangle\\
&|ab'_{21}\rangle=&|0.69 e^{i 254.16^\circ}, 0.59 e^{i45.44^\circ}, 0.41 e^{i 28.36^\circ}, 0.04 e^{i 43.84^\circ} \rangle\\
&|ab'_{22}\rangle=&|0.27 e^{i 70.02^\circ}, 0.16 e^{i18.03^\circ}, 0.20 e^{i 33.61^\circ}, 0.93 e^{i 85.52^\circ} \rangle
\end{eqnarray}
while the self-adjoint operator 
${\cal E}_{AB'}$ is given by
\begin{eqnarray}
{\cal E}_{AB'}&=&|ab'_{11}\rangle\langle ab'_{11}|-|ab'_{12}\rangle\langle ab'_{12}|-|ab'_{21}\rangle\langle ab'_{21}|+|ab'_{22}\rangle\langle ab'_{22}|\nonumber\\
&=&\small
\left( \begin{array}{cccc}
-0.001	&				0.587+0.397i &	0.555+0.434i &	0.035+0.0259i	\\			
0.587-0.397i & -0.489 &	0.497+0.0341i &	-0.106-0.005i \\	
0.555-0.434i & 0.497-0.0341i & -0.503	&	0.045-0.001i \\	0.035-0.0259i &	-0.106+0.005i &	0.045+0.001i & 0.992	\end{array} \right)
\normalsize
\end{eqnarray}
In the third measurement 
$A'B$, the ON basis of eigenvectors is given by
\begin{eqnarray}
&|a'b_{11}\rangle=&|0.44 e^{i 80.05^\circ}, 0.73 e^{i48.97^\circ}, 0.48 e^{i 25.56^\circ}, 0.21 e^{i 274.18^\circ} \rangle\\
&|a'b_{12}\rangle=&|0.02 e^{i 207.84^\circ}, 0.55 e^{i211.53^\circ}, 0.83 e^{i 9.71^\circ}, 0.10 e^{i 208.15^\circ} \rangle\\
&|a'b_{21}\rangle=&|0.89 e^{i 261.73^\circ}, 0.39 e^{i50.87^\circ}, 0.24 e^{i 26.35^\circ}, 0 e^{i 44.62^\circ} \rangle\\
&|a'b_{22}\rangle=&|0.10 e^{i 71.50^\circ}, 0.13 e^{i19.94^\circ}, 0.17 e^{i 37.18^\circ}, 0.97 e^{i 84.23^\circ} \rangle
\end{eqnarray}
while the self-adjoint operator 
${\cal E}_{A'B}$ is given by
\begin{eqnarray}
{\cal E}_{A'B}&=&|a'b_{11}\rangle\langle a'b_{11}|-|a'b_{12}\rangle\langle a'b_{12}|-|a'b_{21}\rangle\langle a'b_{21}|+|a'b_{22}\rangle\langle a'b_{22}| \nonumber\\ 
&=&\small
\left( \begin{array}{cccc}
-0.587 &	0.568+0.353i	&	0.274+0.365i	&	0.002+0.004i \\																							
0.568-0.353i & 0.090	 &				0.681+0.263i & -0,110-0.007i \\				
0.274-0.365i &		0.681-0.263i & -0.484	&	0.150-0.050i \\				
0,002-0.004i & -0,110+0.007i & 0.150+0.050i &	0.981	\end{array} \right)
\normalsize
\end{eqnarray}
In the fourth measurement 
$A'B'$, the ON basis of eigenvectors is given by
\begin{eqnarray}
&|a'b'_{11}\rangle=&| 0.74 e^{i 272.32^\circ}, 0.02 e^{i42.02^\circ}, 0.62 e^{i 38.40^\circ}, 0.26 e^{i 334.31^\circ} \rangle\\
&|a'b'_{12}\rangle=&| 0.02 e^{i 32.17^\circ}, 0.31 e^{i353.95^\circ}, 0.36 e^{i 242.65^\circ}, 0.88 e^{i 356.07^\circ} \rangle\\
&|a'b'_{21}\rangle=&| 0.27 e^{i 278.36^\circ}, 0.87 e^{i65.99^\circ}, 0.31 e^{i 205.22^\circ}, 0.28 e^{i 223.08^\circ} \rangle\\
&|a'b'_{22}\rangle=&| 0.62 e^{i 114.51^\circ}, 0.39 e^{i81.45^\circ}, 0.62 e^{i 65.70^\circ}, 0.29 e^{i 327.92^\circ} \rangle
\end{eqnarray}
while the self-adjoint operator 
$ {\cal E}_{A'B'}$ is given by
\begin{eqnarray}
{\cal E}_{A'B'}&=&|a'b'_{11}\rangle\langle a'b'_{11}-|a'b'_{12}\rangle\langle a'b'_{12}|-|a'b'_{21}\rangle\langle a'b'_{21}|+|a'b'_{22}\rangle\langle a'b'_{22}|\nonumber\\ 
&=&\small
\left( \begin{array}{cccc}
0.854	&				0.385+0.243i & -0.035-0.164i &	-0.115-0.146i \\				
0.385-0.243i & -0.700	&	0.483+0.132i & -0.086+0.212i \\				
-0.035+0.164i &	0.483-0.132i & 0.542 &	0.093+0.647i \\				
-0.115+0.146i &	-0.086-0.212i &	0.093-0.647i &	-0.697	
\end{array} \right)
\normalsize
\end{eqnarray}
Our quantum-theoretic representation in ${\mathbb C}^{4}$ of the entangled measurements in our cognitive experiment is thus completed. In the next section we put forward how our analysis might have far reaching consequences in physics.

\section{Toward an application to EPR-Bell experiments\label{EPR_Bell}}
Our analysis in the previous sections for the conceptual combination {\it The Animal Acts} shows that, whenever Bell's inequalities are violated by experimental data and the marginal distribution law does not hold, then any quantum model in Hilbert space describing these data should include not only entangled states, but also entangled measurements and entangled evolutions, due to this marginal law violation. 
This result is absolutely general, since it does not depend on the fact that conceptual entities are considered, and hence our 
analysis can be successfully applied to physical quantum entities as well. In \cite{aertssozzo2013a,aertssozzo2013b,aerts2013} we analyze different examples and aspects of physical as well as cognitive situations violating Bell's inequalities and the marginal law, and its relation to signalling, and construct explicitly the corresponding entangled measurements.
Moreover, our analysis is also valid for the situation of the EPR-Bell experiments testing nonlocality of natural processes and their compatibility with quantum-mechanical predictions \cite{aspect1982,weihsetal1998}. In these experiments, one typically considers a pair of quantum particles (e.g., photons) prepared in a suitable entangled state (e.g., the singlet spin state) and travelling in opposite directions. Each particle is then measured by entering a suitably oriented apparatus (e.g., polarizer) that performs a (e.g., polarization) measurement. This means that a product measurement is considered in which the two measuring apparatuses are spatially separated. The obtained results are then statistically analyzed and the violation of Bell's inequalities is finally considered as an experimental confirmation of nonlocality.
What is implicitly assumed is that, since experimental data seem to agree with the predictions of quantum mechanics and since only product measurements are considered in these EPR-Bell tests, the marginal distributions of two sets of these measurements should coincide, as a consequence of the the measurements being products. This point can however be questioned. Indeed, different experiments, including, in particular, Aspect's \cite{aspect1982} and Weihs' \cite{weihsetal1998} showed that {\it anomalies} appear in the marginal distributions of the measurements that are considered. These anomalies are at the level of probabilities and not at the level of expectation values (which actually enter Bell's inequalities) and have been analyzed pointing out that quantum predictions could fail in these EPR-Bell tests \cite{ak1,ak2}. 

Coming to our results in this paper, we explicitly want to put forward the possible hypothesis that the anomalies above could have a different origin and instead suggest a `stronger form of entanglement' also including measurements and evolutions. Indeed, the observed violation of the marginal distribution law at least invoke the intriguing possibility that measurements and evolutions are entangled, and that the operator-linked entanglement, which mathematically appears in a natural way in standard quantum theory, is also present in physics, even in the archetypical entanglement situation of coupled spins.  If this would be the case, we can apply the new entanglement scheme we developed in this paper to deal with entanglement in polarization measurements. Indeed, let us suppose that data are collected for a given standard EPR-Bell experiment and that non-negligible violations of marginal law are observed. Following our investigation of the quantum representation of our cognitive example, the following possibilities then open up the following possibilities: 

(i) we can work out a quantum representation in ${\mathbb C}^{4}$, assuming that the initial state is the singlet spin state; (ii) we can assume that four non-product self-adoint operators representing the entangled measurements exist that fit experimental data in the singlet spin state; (iii) we can provide an explicit representation of these measurements in terms of non-product self-adjoint operators; (iv) we can quantify the degree of entanglement of these measurements evaluating their deviations from product self-adjoint operators in ${\mathbb C}^2\otimes {\mathbb C}^2$; (v) we can study the connections between parameter dependence, {\it no-signaling condition} and the compatibility with relativity theory \cite{peres2004,gisin2006}.
 
While admitting that what proceeds is very challenging, we 
do not put forward a complete treatment of this `entanglement in the EPR-Bell experiments' in the present paper. But, we plan to investigate this potentially relevant aspect of quantum physics in the next future in detail. If the speculative hypothesis of the presence of entangled measurements also in physics would turn out to be true, it would constitute a first case of how insights gained in quantum cognition influence quantum physics in the physical realm.

\end{document}